\documentclass[twoside,11pt]{article}

\usepackage{blindtext}

\usepackage[preprint]{jmlr2e}

\DeclareMathOperator*{\argmax}{arg\,max}

\usepackage{lastpage}
\jmlrheading{27}{2025}{1-\pageref{LastPage}}{10/25}{}{}{Killian Steunou, Th\'eo Druilhe and Sigurd Saue}

\ShortHeadings{Sparse Representations Improve Adversarial Robustness of Neural Network Classifiers}{Steunou, Druilhe and Saue}
\firstpageno{1}

\begin{document}

\title{Sparse Representations Improve Adversarial Robustness of Neural Network Classifiers}

\author{\name Killian Steunou\thanks{Lead and corresponding author. Primary contributor.} \email killian.steunou@ens-paris-saclay.fr \\
        École Normale Supérieure Paris-Saclay \\
        4 Av. des Sciences, 91190 Gif-sur-Yvette, France
        \AND
       \name Th\'eo Druilhe \email theo.druilhe@ut-capitole.fr \\
       \name Sigurd Saue \email sigurd.saue@ut-capitole.fr \\
       % \addr Independent researchers\\
       Toulouse School of Economics \\ %Universit\'e Toulouse Capitole\\
       21 allée de Brienne, 31000, Toulouse, France}

\editor{My editor}

\maketitle

\begin{abstract}%   <- trailing '%' for backward compatibility of .sty file
Deep neural networks perform remarkably well on image classification tasks but remain vulnerable to carefully crafted adversarial perturbations. This work revisits linear dimensionality reduction as a simple, data-adapted defense. We empirically compare standard Principal Component Analysis (PCA) with its sparse variant (SPCA) as front-end feature extractors for downstream classifiers, and we complement these experiments with a theoretical analysis. On the theory side, we derive exact robustness certificates for linear heads applied to SPCA features: for both $\ell_\infty$ and $\ell_2$ threat models (binary and multiclass), the certified radius grows as the dual norms of $W^\top u$ shrink, where $W$ is the projection and $u$ the head weights. We further show that for general (non-linear) heads, sparsity reduces operator-norm bounds through a Lipschitz composition argument, predicting lower input sensitivity. Empirically, with a small non-linear network after the projection, SPCA consistently degrades more gracefully than PCA under strong white-box and black-box attacks while maintaining competitive clean accuracy. Taken together, the theory identifies the mechanism (sparser projections reduce adversarial leverage) and the experiments verify that this benefit persists beyond the linear setting. Our code is available at \href{https://github.com/killian31/SPCARobustness}{https://github.com/killian31/SPCARobustness}.
\end{abstract}

\begin{keywords}
  adversarial robustness, sparse representations, neural network classifiers, robust feature learning, regularization
\end{keywords}

\section{Introduction}
\label{sec:intro}

Machine learning models, particularly deep neural networks, have demonstrated remarkable performance across various domains, from image classification to natural language processing. However, these models exhibit a concerning vulnerability to adversarial examples, carefully crafted perturbations that, while often imperceptible to humans, can cause classifiers to make incorrect predictions with high confidence. This vulnerability poses significant challenges for deploying machine learning systems in security-critical applications such as autonomous vehicles, medical diagnostics, and fraud detection. The widespread nature of this vulnerability across different model architectures has sparked extensive research into defense mechanisms. Proposed approaches include adversarial training, robust optimisation, regularisation techniques, and input transformation methods. Despite these efforts, the fundamental challenge remains: creating models that maintain high accuracy on clean data while demonstrating resilience against adversarial manipulations.

In this paper we explore dimensionality reduction as a defense and ask whether enforcing sparsity in the projection improves robustness. Specifically, we investigate Sparse Principal Component Analysis (SPCA) from \cite{spca} as a data-adapted linear front-end in place of standard PCA \citep{pca}. Our hypothesis builds on the view that adversarial examples exploit high-dimensional, non-robust features that correlate with labels in training data but fail to generalize under small perturbations. By restricting each component to depend on only a small subset of input dimensions, SPCA may filter out these non-robust directions and emphasise more stable structure.

Beyond empirical evaluation, we provide a theoretical account of why sparsity helps. We derive exact robustness certificates for linear heads operating on SPCA features. In the $\ell_\infty$ and $\ell_2$ settings, both binary and multiclass, the prediction is provably invariant whenever the clean margin exceeds a term proportional to $\|W^\top u\|_{1}$ or $\|W^\top u\|_{2}$, respectively, where $W$ is the projection and $u$ are head weights. Because SPCA drives many entries of $W$ to zero and shrinks its column norms, these dual norms decrease, enlarging the certified radius. For general non-linear heads, we bound the end-to-end sensitivity via a Lipschitz composition $\|C_\phi \circ W\|_{p\to 2} \le L_C \|W\|_{p\to 2}$, showing that sparsity in $W$ tightens operator-norm bounds and thus reduces the model’s worst-case input sensitivity.

The contributions of our work are as follows:
\begin{itemize}
    \item A systematic comparison between PCA and SPCA as feature extractors for classification under adversarial conditions on MNIST and a CIFAR-binary task, covering white-box ($\ell_\infty$, $\ell_2$) and black-box attacks.
    \item A theoretical analysis establishing exact robustness certificates for linear heads on top of SPCA features (binary and multiclass, $\ell_\infty$ and $\ell_2$), and Lipschitz-based sensitivity bounds for general non-linear heads.
    \item An empirical finding that the sparsity mechanism predicted by theory persists with a small non-linear network: SPCA-based classifiers degrade more gracefully than PCA-based ones while maintaining competitive clean accuracy.
    \item An examination of sparsity-robustness trade-offs, relating the degree of sparsity in $W$ to the resulting robustness profile.
\end{itemize}

Our results suggest that carefully designed sparse feature extraction can serve as an effective and interpretable defense mechanism against adversarial attacks, and that the mechanism has principled support: sparsity contracts the dual/operator norms that govern worst-case changes, thereby limiting adversarial leverage. We next situate our approach within related work before formalising the problem and presenting our theoretical analysis and experiments.

\section{Related Work}
\label{sec:related}

A large body of work investigates why modern classifiers are brittle and how to defend them. One perspective attributes adversarial vulnerability to the exploitation of non-robust features, i.e., high-dimensional but brittle correlations that drive accuracy on clean data while being easy to manipulate \citep{ilyas2019adversarialnotbugs}. This motivates defenses that either (i) project inputs onto lower-complexity representations to discard non-robust directions, or (ii) change the classifier family altogether (e.g., generative classifiers) so that decisions align with data likelihood. Our work follows the first route via linear dimensionality reduction with sparsity, and we position it relative to both lines below.

Prior defenses reduce the adversary’s search space by transforming inputs before classification. Representative examples include bit-depth reduction, JPEG compression, total-variation minimization, and image quilting \citep{guo2018countering}, as well as feature squeezing (bit-depth and spatial smoothing) used for detection \citep{xu2017featureSqueezing}. These methods share the spirit of dimensionality/complexity reduction but typically operate as non-learned image transforms. In contrast, we learn a data-adapted linear projection and ask whether imposing sparsity on this projection improves robustness over standard PCA.

High-dimensional models like deep neural networks are famously vulnerable to small adversarial perturbations. A common defense strategy is to reduce input dimensionality or project inputs onto a lower-dimensional “manifold” of normal data, with the goal of discarding adversarial noise. Prior work has shown promise for such approaches. For example, \cite{bhagoji2017enhancingrobustnessmachinelearning} explored data transformations including PCA as a preprocessing defense, reporting that projecting inputs onto a subspace of top principal components can significantly increase the effort required for a successful attack.  This supports the intuition that many adversarial perturbations lie in directions of input space that can be truncated or removed through dimensionality reduction.

Researchers have also analysed adversarial examples through the lens of principal components. \cite{hendrycks2017earlymethodsdetectingadversarial} observed that adversarial images exhibit different PCA spectra than genuine images. In particular, adversarial inputs tend to have abnormally large variance along some principal directions. \cite{jere2019principalcomponentpropertiesadversarial} extended this analysis across multiple architectures, showing that small adversarial perturbations consistently alter an input’s projection on principal components in characteristic ways. Such findings motivated the use of PCA as a defensive filter; for instance, \cite{bhagoji2017enhancingrobustnessmachinelearning} demonstrated that removing low-variance PCA components of input data can improve robustness against $\ell_2$ and $\ell_\infty$ norm attacks, forcing attackers to exert significantly more distortion to fool the classifier.

Sparse Principal Component Analysis is a variant of PCA that seeks a projection onto principal components with sparse support, i.e. each component is a linear combination of only a small subset of original features \citep{spca}.  By selecting a subset of important features, SPCA can yield more interpretable and possibly more robust representations, since spurious noise spread across many features might be excluded.  Despite SPCA’s popularity in high-dimensional data analysis, its use as a defense against adversarial attacks remains largely unexplored.  To our knowledge, no prior empirical study has focused on applying SPCA to improve classifier resistance to adversarial perturbations in image recognition tasks. Related work has predominantly centred on standard PCA or non-sparse transformations, as discussed above.

There are a few recent developments that hint at the potential of SPCA-like approaches for robustness. \cite{awasthi2020estimatingprincipalcomponentsadversarial} theoretically formulated a “robust PCA” objective for finding low-dimensional data representations that are maximally stable to bounded adversarial perturbations.  Interestingly, their formulation includes a sparse PCA criterion as a special case, and they devise algorithms to compute a robust subspace with provable guarantees.  A complementary theoretical analysis of PCA’s susceptibility to worst-case perturbations was given by \cite{li2020intheadversarialrobustnesspca}, who studied optimal modification strategies that maximize the distance between subspaces learned from clean versus perturbed data.  Their results provide formal insights into how adversaries can exploit the structure of the data matrix to destabilise PCA projections, especially when the adversary can manipulate the entire dataset. \cite{dorsi2020sparsepcaalgorithmsadversarial} likewise study the problem of identifying the top sparse principal components in a dataset that has been adversarially corrupted, providing the first algorithms with certified resilience to small adversarial changes in the input covariance.  These works are theoretical and focus on algorithmic recovery guarantees rather than end-to-end classifier defense, but they underscore the intuition that enforcing sparsity in feature projections can enhance robustness.

In comparison to previous PCA-based defenses, our study is novel in explicitly leveraging Sparse PCA as a defense mechanism for adversarial robustness in high-dimensional image classification.  By constraining principal components to be sparse, we aim to concentrate the model’s attention on a core subset of stable features, potentially pruning away the noisy directions adversaries exploit.  This approach bridges a gap in the literature: whereas earlier defenses either used dense PCA or non-linear autoencoders to reduce input dimensionality, we investigate whether a sparse linear projection can provide a favourable trade-off between preserving pertinent information and removing adversarial noise. In summary, our work builds on the idea of dimensionality reduction for robustness but introduces and evaluates SPCA in this context for the first time, showing how it compares to prior PCA-based and manifold-projection defenses.

Orthogonal to dimensionality-reduction defenses, \cite{li2019aregenerative} propose deep generative classifiers that model $p(x,y)$ and use likelihood for both prediction and attack detection. Their experiments cover MNIST and a CIFAR-binary task constructed from CIFAR-10 (airplane vs.\ frog), and report improved robustness against common attacks. Conceptually, their robustness arises from rejecting off-manifold inputs via low likelihood, whereas our approach aims to prevent vulnerability by discarding non-robust directions through a sparse linear projection learned from data. The two lines are complementary: SPCA can be used as a front-end to generative or discriminative models, and our evaluation protocol (MNIST/CIFAR-binary; FGSM/PGD/MIM; $\ell_\infty$ and $\ell_2$) mirrors \cite{li2019aregenerative} in spirit but targets a different mechanism.

The view that adversaries exploit non-robust features \citep{ilyas2019adversarialnotbugs} provides a rationale for favoring sparse representations that emphasise a small, stable subset of input dimensions. Robust-subspace formulations for PCA provide algorithmic backing for this intuition \citep{awasthi2020estimatingprincipalcomponentsadversarial}, while certified approaches like randomized smoothing \citep{cohen2019certifiedsmoothing} pursue orthogonal guarantees. Our findings suggest that enforcing sparsity already yields meaningful empirical robustness improvements and can potentially complement certified or generative defenses.

\section{Problem Statement}
\label{sec:problem_statement}

\paragraph{Data and Notation.}  Let $\{(x_i, y_i)\}_{i=1}^N$ be a dataset of $N$ samples, where each input $x_i \in \mathbb{R}^D$ and label $y_i \in \mathcal{Y}$, with $\mathcal{Y} = \{1, 2, \dots, K\}$. Our goal is to learn a function $h:\mathbb{R}^D \rightarrow \{1, \dots, K\}$ that predicts the correct label $y_i$ from $x_i$ while also being robust to adversarial perturbations.

\paragraph{Feature Extractor.}  We introduce a dimensionality-reduction map $\Phi: \mathbb{R}^D \rightarrow \mathbb{R}^r$, where $r < D$, that transforms an input $x \in \mathbb{R}^D$ into a lower‑dimensional representation $\Phi(x) \in \mathbb{R}^r$.  The matrix of parameters for this map is denoted by $W \in \mathbb{R}^{r \times D}$, and we optionally include a bias vector $b \in \mathbb{R}^r$. Hence,
$$
\Phi(x) \;=\; W\,x \;+\; b.
$$
In this work, $W$ (and possibly $b$) is computed by either Principal Component Analysis (PCA) or Sparse Principal Component Analysis (SPCA).

\subparagraph{PCA.}  Given a set of training inputs $\{x_i\}_{i=1}^N$, let
$$
S \;=\; \frac{1}{N}\sum_{i=1}^N (x_i - \bar{x})(x_i - \bar{x})^\top
$$
be the empirical covariance matrix, where $\bar{x}$ is the empirical mean of the data.  PCA seeks $W \in \mathbb{R}^{r \times D}$ that solves:
\begin{align*}
\max_{W}\; & \operatorname{trace}\bigl(W S W^\top \bigr) \\
\text{subject to}\; & W W^\top = I_r,
\end{align*}
where $I_r$ is the $r \times r$ identity matrix.  In practice, $W$ is chosen as the top $r$ eigenvectors of $S$, and $b$ is chosen so that the transform is mean‑centred.

\subparagraph{SPCA.}  Sparse Principal Component Analysis imposes additional sparsity constraints on $W$.  A common formulation is:
\begin{align*}
    \max_{W}\; & \operatorname{trace}\bigl(W S W^\top\bigr) \\
\text{subject to} \; & \|W_j\|_2 = 1\quad\text{for each row }W_j, \\
& \|W\|_{1} \leq \alpha,
\end{align*}

where $\|W\|_1$ denotes the sum of absolute values of all entries (or a row‑wise/grouped variant), and $\alpha>0$ is a tuning parameter controlling sparsity.  This constraint forces each principal component to have many coefficients set to or near zero, thereby increasing interpretability and potentially affecting robustness.

\paragraph{Classifier.}  Let $C_\phi: \mathbb{R}^r \rightarrow \{1, \dots, K\}$ be a parametric classifier (e.g., a neural network) with parameters $\phi$.  The end‑to‑end predictor $h$ is thus given by
$$
h(x)\;=\;C_\phi\bigl(\Phi(x)\bigr)
\;=\;
C_\phi\bigl(Wx + b\bigr).
$$
We train $h$ by minimising a classification loss (e.g., cross‑entropy) over the dataset:
$$
\min_{\phi}\,\frac{1}{N}\sum_{i=1}^N \ell\Bigl(C_\phi\bigl(\Phi(x_i)\bigr),\,y_i\Bigr),
$$
where $\ell(\cdot,\cdot)$ is an appropriate loss function.

\paragraph{Adversarial Attack.}  An adversarial example is a perturbed input $x'_i$ such that $x'_i$ is close to $x_i$ in some metric but leads the classifier $h$ to an incorrect prediction.  Formally, a threat model specifies a norm‑based constraint, such as $\|x'_i - x_i\|_\infty \le \varepsilon$, and an adversary seeks to solve:
$$
x'_i \;=\;\argmax_{x'\text{ s.t. }\,\|x' - x_i\|_\infty \le \varepsilon} \ell\Bigl(h(x'),\;y_i\Bigr).
$$
The adversarial robustness of $h$ at perturbation level $\varepsilon$ is then measured by the accuracy of $h$ on the set of adversarial examples $\{x'_i\}_{i=1}^N$.

\paragraph{Objective.}  Our aim is to compare the robustness of two classifiers: one built upon the PCA‑based feature extractor $W_{\mathrm{PCA}}$, and another on the SPCA‑based extractor $W_{\mathrm{SPCA}}$. Specifically, we wish to investigate whether enforcing sparsity in the linear transformation $\Phi$ leads to improved adversarial robustness. Hence, we study $\mathrm{Robustness}(h_{\mathrm{PCA}}, \varepsilon)$
vs.
$\mathrm{Robustness}(h_{\mathrm{SPCA}}, \varepsilon)$, for various perturbation magnitudes $\varepsilon$, where
\begin{align*}
    h_{\mathrm{PCA}}(x) &= C_\phi\bigl(W_{\mathrm{PCA}}\,x + b_{\mathrm{PCA}}\bigr), \\
    h_{\mathrm{SPCA}}(x) &= C_{\phi'}\bigl(W_{\mathrm{SPCA}}\,x + b_{\mathrm{SPCA}}\bigr).
\end{align*}

In the subsequent sections we provide theoretical and empirical evidence that sparsity in the feature extraction matrix $W$ can mitigate the effects of adversarial perturbations, thereby yielding more robust classifiers.

\section{Theoretical Analysis}

The empirical results presented later in this paper demonstrate that Sparse
Principal Component Analysis (SPCA) enhances robustness compared to standard
PCA. In this section we provide formal guarantees that support this observation.
We begin by fixing notation and the threat model, then derive exact robustness
certificates for linear heads, extend them to multiclass settings, and provide
Lipschitz-based bounds for general non-linear heads. We also show explicitly how
sparsity in the projection matrix improves the certificates.

\subsection{Setup and Notation}

Let $W \in \mathbb{R}^{r \times D}$ denote the projection matrix learned by PCA
or SPCA, and $z = Wx \in \mathbb{R}^r$ the reduced representation of an input
$x \in \mathbb{R}^D$. We write $\|\cdot\|_p$ for the $\ell_p$ norm and
$\|\cdot\|_{p \to q}$ for the operator norm induced by $\ell_p$ on the domain and
$\ell_q$ on the codomain.

We consider additive adversaries that may perturb an input within a norm ball:
for a radius $\varepsilon > 0$, the adversary can replace $x$ by $x+\delta$ with
$\|\delta\|_p \le \varepsilon$, where $p \in \{2,\infty\}$ throughout this section.

\subsection{Certified Robustness for Binary Linear Heads}

Consider a binary linear classifier on top of the projection,
\[
f(x) \;=\; \mathrm{sign}\!\big(u^\top z + b\big)
\;=\; \mathrm{sign}\!\big(u^\top W x + b\big),
\]
where $u \in \mathbb{R}^r$ and $b \in \mathbb{R}$. For a labeled example
$(x,y)$ with $y \in \{-1,1\}$, define the (signed) margin
$m(x) = y\,(u^\top W x + b)$.

\begin{theorem}[Exact $\ell_\infty$ robustness certificate]
\label{thm:binary-linf}
If $m(x) > \varepsilon\,\|W^\top u\|_1$, then for all perturbations
$\delta$ with $\|\delta\|_\infty \le \varepsilon$ the prediction is invariant:
$f(x+\delta) = f(x)$.
\end{theorem}

\begin{proof}
For any such $\delta$,
\(
\big|u^\top W \delta\big|
\le \varepsilon\,\|W^\top u\|_1
\)
by $\ell_\infty$-$\ell_1$ duality. Hence the logit $y(u^\top W (x+\delta)+b)$
remains positive if $m(x)>\varepsilon\|W^\top u\|_1$.
\end{proof}

\begin{theorem}[Exact $\ell_2$ robustness certificate]
\label{thm:binary-l2}
If $m(x) > \varepsilon\,\|W^\top u\|_2$, then for all perturbations
$\delta$ with $\|\delta\|_2 \le \varepsilon$ the prediction is invariant:
$f(x+\delta) = f(x)$.
\end{theorem}

\begin{proof}
By Cauchy-Schwarz, $\big|u^\top W \delta\big| \le \|W^\top u\|_2\,\|\delta\|_2
\le \varepsilon\,\|W^\top u\|_2$.
\end{proof}

\paragraph{Certified radii.}
Theorems~\ref{thm:binary-linf} and~\ref{thm:binary-l2} yield per-example
certified radii
\[
\varepsilon^\star_{\infty}(x)
= \frac{m(x)}{\|W^\top u\|_1},
\qquad
\varepsilon^\star_{2}(x)
= \frac{m(x)}{\|W^\top u\|_2}.
\]
Smaller dual norms of $W^\top u$ lead to larger certified radii.

\subsection{Multiclass Linear Heads}

Let $U = [u_1,\dots,u_K] \in \mathbb{R}^{r \times K}$ be class weight vectors and
$b_1,\ldots,b_K \in \mathbb{R}$. The predicted class is
\(
f(x) = \arg\max_{k} \{ u_k^\top W x + b_k \}.
\)
Let $k^\star$ denote the argmax for the clean input.

\begin{theorem}[Multiclass $\ell_\infty$ certificate]
\label{thm:multi-linf}
For each $k \neq k^\star$, define the pairwise margin
\(
\gamma_k(x) = (u_{k^\star}-u_k)^\top W x + (b_{k^\star}-b_k).
\)
If
\(
\gamma_k(x) > \varepsilon\,\|W^\top (u_{k^\star}-u_k)\|_1
\)
for all $k \neq k^\star$, then $f(x+\delta)=k^\star$ for all
$\delta$ with $\|\delta\|_\infty \le \varepsilon$.
\end{theorem}

\begin{proof}
For each competitor $k$, the perturbed logit difference equals
\(
(u_{k^\star}-u_k)^\top W (x+\delta) + (b_{k^\star}-b_k)
\),
whose change due to $\delta$ is bounded in magnitude by
$\varepsilon \|W^\top (u_{k^\star}-u_k)\|_1$ by duality. If the clean margin
exceeds this bound for all $k$, all pairwise differences remain positive.
\end{proof}

\begin{theorem}[Multiclass $\ell_2$ certificate]
\label{thm:multi-l2}
If
\(
\gamma_k(x) > \varepsilon\,\|W^\top (u_{k^\star}-u_k)\|_2
\)
for all $k \neq k^\star$, then $f(x+\delta)=k^\star$ for all
$\delta$ with $\|\delta\|_2 \le \varepsilon$.
\end{theorem}

\begin{proof}
As in Theorem~\ref{thm:binary-l2}, by Cauchy-Schwarz the change in each pairwise
logit difference is at most $\varepsilon\,\|W^\top (u_{k^\star}-u_k)\|_2$.
\end{proof}

\subsection{General Lipschitz Heads and Operator-Norm Bounds}

Let $C_\phi : \mathbb{R}^r \to \mathbb{R}^K$ denote a (possibly non-linear)
classifier with Lipschitz constant $L_C$ under $\ell_2$, i.e.,
$\|C_\phi(z)-C_\phi(z')\|_2 \le L_C \|z-z'\|_2$ for all $z,z'$. Then the
composition $x \mapsto C_\phi(Wx)$ has Lipschitz constant
\begin{equation}
\label{eq:lip-compose}
\|C_\phi \circ W\|_{p \to 2}
\;\le\;
L_C\,\|W\|_{p \to 2}.
\end{equation}
In particular,
\[
\|W\|_{\infty \to 2}
\;\le\;
\sum_{j=1}^D \|w_j\|_2
\quad\text{and}\quad
\|W\|_{\infty \to 2}
\;\le\;
\sqrt{D}\,\|W\|_{2 \to 2},
\]
where $w_j$ is the $j$-th column of $W$. Moreover,
\(
\max_j \|w_j\|_2 \le \|W\|_{\infty \to 2}
\),
since taking an $\ell_\infty$-unit vector supported on coordinate $j$ yields
$W e_j = w_j$. Thus the $\ell_\infty\!\to\!\ell_2$ operator norm's value is
between the maximum and the sum of column norms and is also controlled by
$\|W\|_{2\to 2}$. Any structural constraint that reduces column norms or spectral
norm therefore tightens the worst-case sensitivity bound in
\eqref{eq:lip-compose}.

\subsection{How Sparsity in \texorpdfstring{$W$}{W} Improves Certificates}

We now quantify how sparsity decreases the dual norms that appear in the
certificates.

\begin{lemma}[Dual-norm control via column norms]
\label{lem:dual-col}
For any $u \in \mathbb{R}^r$,\\
$\|W^\top u\|_1 \le \|u\|_2 \sum_{j=1}^D \|w_j\|_2$
and
$\|W^\top u\|_2 \le \|u\|_2\,\|W\|_{2\to 2}.$
\end{lemma}

\begin{proof}
Write $(W^\top u)_j = u^\top w_j$. Then
$\|W^\top u\|_1
= \sum_{j=1}^D |u^\top w_j|
\le \sum_{j=1}^D \|u\|_2\,\|w_j\|_2
= \|u\|_2 \sum_{j=1}^D \|w_j\|_2$
by Cauchy-Schwarz. The $\ell_2$ bound is the definition of the operator norm
$\|W\|_{2\to 2}$ applied to $u$.
\end{proof}

Lemma~\ref{lem:dual-col} shows that making many entries of $W$ exactly zero
(SPCA) reduces the column $\ell_2$ norms and thus decreases both
$\|W^\top u\|_1$ and $\|W^\top u\|_2$. Consequently, the certified radii in
Theorems~\ref{thm:binary-linf}-\ref{thm:multi-l2} increase. This formalizes the
intuition that sparsity removes vulnerable directions that dense PCA preserves.

\paragraph{Practical computation.}
For linear heads, the per-example certified radii
$\varepsilon^\star_\infty(x)$ and $\varepsilon^\star_2(x)$ are computable by a
single forward pass to obtain $m(x)$ and a matrix-vector product to evaluate
$\|W^\top u\|_1$ or $\|W^\top u\|_2$. For multiclass models, one evaluates the
pairwise margins and dual norms in Theorems~\ref{thm:multi-linf}-\ref{thm:multi-l2}.
Reporting the fraction of test points certified at each radius (``certified
accuracy'') provides a direct, attack-independent robustness comparison between
PCA and SPCA.

\paragraph{Scope and connection to practice.}
Theorems~1--4 provide exact robustness certificates for linear heads operating on SPCA features.
Our main experiments deliberately use a small non-linear MLP after the projection to test whether the sparsity-driven effect persists beyond linear models.
While the linear-head assumption does not hold exactly there, the same mechanism remains: sparsity in $W$ shrinks dual/operator norms that govern input sensitivity.
For general heads $C_\phi$ with Lipschitz constant $L_C$ in feature space, the composition bound
$\|C_\phi \circ W\|_{p \to 2} \le L_C \, \|W\|_{p \to 2}$
shows that sparsifying $W$ tightens worst-case sensitivity even without linearity.
The empirical sections confirm that the robustness advantage of SPCA over PCA is visible with the small non-linear classifier as well.

\section{Experiments}
\label{sec:exp}

This section describes our experimental setup and reports empirical results comparing PCA‑ and SPCA‑based classifiers under a suite of adversarial attacks.  We evaluate both white‑box attacks (where the attacker has full access to the model and its gradients) and black‑box attacks (where only model outputs can be queried).

\subsection{Datasets and Preprocessing}

\paragraph{MNIST.}  The first dataset is the widely used MNIST handwritten digits benchmark.  It contains 60,000 training images and 10,000 test images of grayscale digits normalized to $28\times 28$ pixels \citep{deng2012mnist}. We follow the conventional split, using the entire training set to learn the PCA or SPCA projection matrix and to train a downstream classifier, and evaluating robustness on the 10,000‑image test set. After flattening each image into a 784‑dimensional vector and mean‑centring, we project onto $r$ principal components. For SPCA we use a sparsity parameter that retains approximately $5\%$ of the weights in each component.

\paragraph{CIFAR‑Binary.}  Our second dataset is a binary version of CIFAR‑10 \citep{krizhevsky2009learningcifar} inspired by \cite{li2019aregenerative}.  CIFAR‑10 contains 60,000 $32\times32$ color images in 10 classes, with 50,000 training examples and 10,000 test examples. To create the binary variant we extract only the classes ``airplane'' and ``frog'' from the original dataset, yielding 6,000 airplane images and 6,000 frog images.  We split these into 10,000 training examples and 2,000 test examples.  Each image is converted to grayscale, flattened into a 1024‑dimensional vector, mean‑centred, and projected onto $r$ principal components with PCA or SPCA.

\subsection{Model Architecture}

To focus on the effect of the feature extractor rather than the classifier capacity, we employ a small but expressive neural network for classification in both settings.  After projecting the input into $r$ dimensions, we feed the result into a fully connected network with two hidden layers of sizes 256 and 128, each followed by ReLU activation.  A final linear layer maps to the class logits.  Models are trained for 20 epochs using the Adam optimizer with learning rate $1\times10^{-3}$ and batch size 128. For each value of $r$ we train separate PCA‑ and SPCA‑based classifiers to ensure a fair comparison. Both classifiers achieve near‑identical accuracy on clean test data.

\subsection{Attack Models}

We evaluate robustness against four adversarial attacks spanning different norms and threat models.

\paragraph{Fast Gradient Sign Method (FGSM).}  FGSM is a one‑step, white‑box attack that perturbs an input in the direction of the sign of the loss gradient \citep{fgsm}. An adversarial example $x'$ is generated by $x' = x + \varepsilon \cdot \operatorname{sign}(\nabla_x \ell(h(x), y))$, where $\varepsilon$ is the attack strength.  Despite its simplicity, FGSM effectively reduces classification accuracy even for very small $\varepsilon$.

\paragraph{Projected Gradient Descent (PGD).}  PGD is a multi‑step extension of FGSM that iteratively takes gradient steps and projects back onto the allowed perturbation set \citep{pgd}. For perturbation budget $\varepsilon$ and step size $\alpha$ we perform $T$ iterations:
\begin{align*}
x^{(0)} &= x + \mathcal{U}[-\varepsilon, \varepsilon],\\
x^{(t+1)} &= \Pi_{x + \mathcal{B}_p(\varepsilon)}\bigl(x^{(t)} + \alpha \cdot \operatorname{sign}(\nabla_{x^{(t)}} \ell(h(x^{(t)}), y))\bigr),
\end{align*}
where $\Pi$ denotes projection onto the $\ell_p$ ball of radius $\varepsilon$ around $x$.  PGD is often considered a universal first‑order adversary since it reliably finds high‑loss points in the inner maximisation problem of adversarial robustness.  In our experiments we run 40 PGD steps with $\alpha = \varepsilon/4$.

\paragraph{Momentum Iterative Method (MIM).}  MIM augments iterative gradient attacks with a momentum term to stabilise updates and improve transferability \citep{mim}. At each iteration the gradient is accumulated into a velocity vector $g^{(t)}$ and normalized.  The update rule is $g^{(t+1)} = \mu \cdot g^{(t)} + \frac{\nabla_{x^{(t)}} \ell(h(x^{(t)}), y)}{\|\nabla_{x^{(t)}} \ell(h(x^{(t)}), y)\|_1}$, and the adversarial example is updated as $x^{(t+1)} = x^{(t)} + \alpha \cdot \operatorname{sign}(g^{(t+1)})$ followed by projection.  We use a momentum parameter $\mu = 1.0$ and step size $\alpha = \varepsilon/5$ for 20 iterations.

\paragraph{Square Attack.}  The Square Attack is a score‑based black‑box attack that queries only the model’s output probabilities and randomly perturbs square regions of the image until misclassification occurs \citep{andriushchenko2020squareattackqueryefficientblackbox}. The method is query‑efficient and circumvents gradient masking by not requiring access to model gradients. We adopt the implementation of \cite{andriushchenko2020squareattackqueryefficientblackbox} with 5,000 queries and report the success rate as a function of $\varepsilon$.

For all attacks we evaluate two norms: the $\ell_\infty$ norm, which constrains the maximum absolute change to any pixel, and the $\ell_2$ norm, which constrains the overall energy of the perturbation.  Attack strength $\varepsilon$ varies from 0.01 to 0.20 in increments of 0.01.

\subsection{Results}

\subsubsection{White‑box Attacks ($\ell_\infty$)}

Figure~\ref{fig:linf_results} summarises accuracy versus attack strength for FGSM, PGD and MIM under $\ell_\infty$ constraints.  Each plot uses dots to indicate accuracy for PCA (dashed lines) and SPCA (solid lines) as the number of retained components varies from 100 to 200; colors encode the number of components.  The top row shows results on MNIST, the bottom row on CIFAR‑Binary.

\begin{figure*}[!ht]
    \centering
    \begin{subfigure}[b]{0.32\linewidth}
        \includegraphics[width=\linewidth]{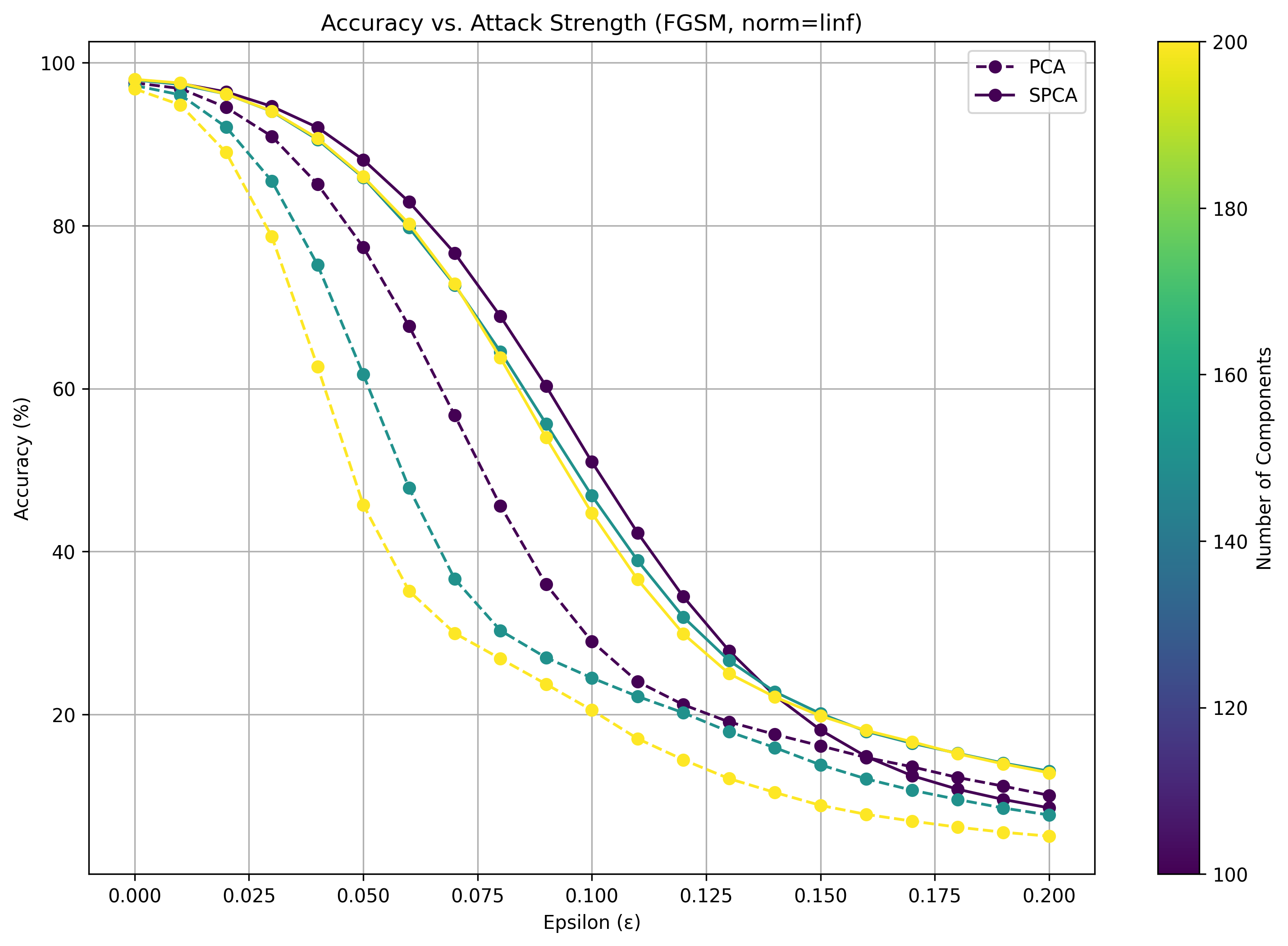}
        \caption{MNIST -- FGSM$_\infty$}
    \end{subfigure}
    \begin{subfigure}[b]{0.32\linewidth}
        \includegraphics[width=\linewidth]{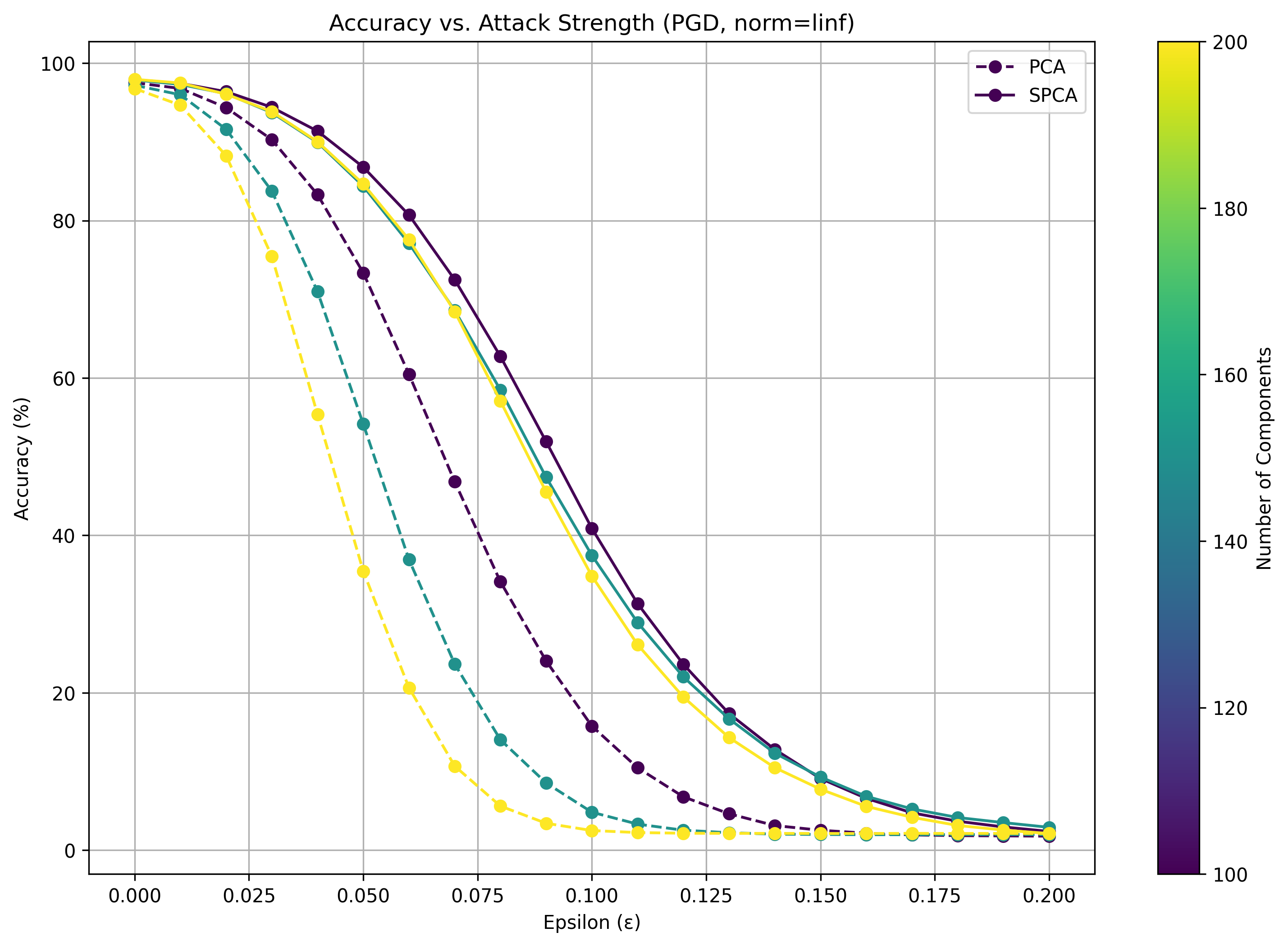}
        \caption{MNIST -- PGD$_\infty$}
    \end{subfigure}
    \begin{subfigure}[b]{0.32\linewidth}
        \includegraphics[width=\linewidth]{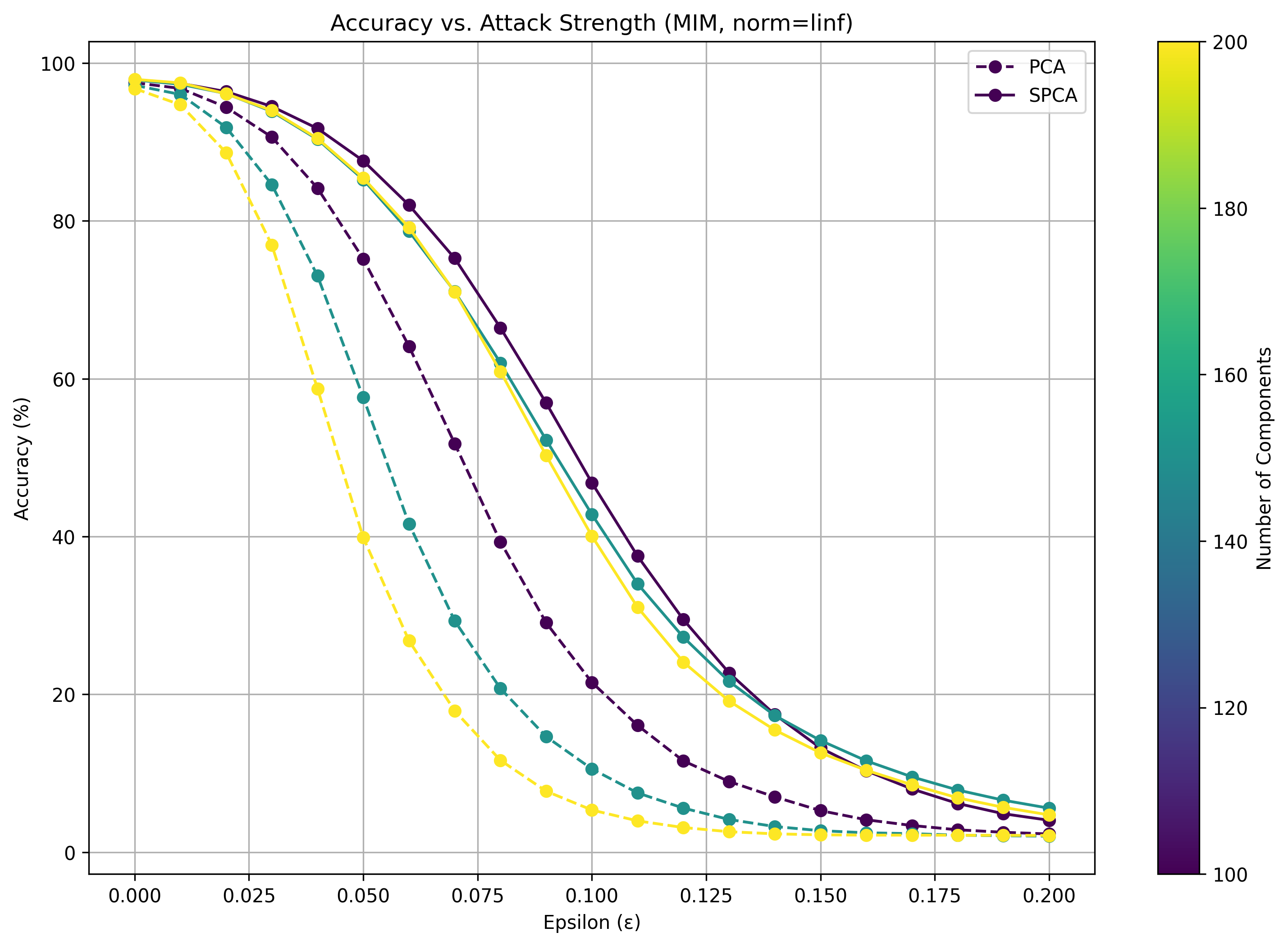}
        \caption{MNIST -- MIM$_\infty$}
    \end{subfigure}
    \\
    \begin{subfigure}[b]{0.32\linewidth}
        \includegraphics[width=\linewidth]{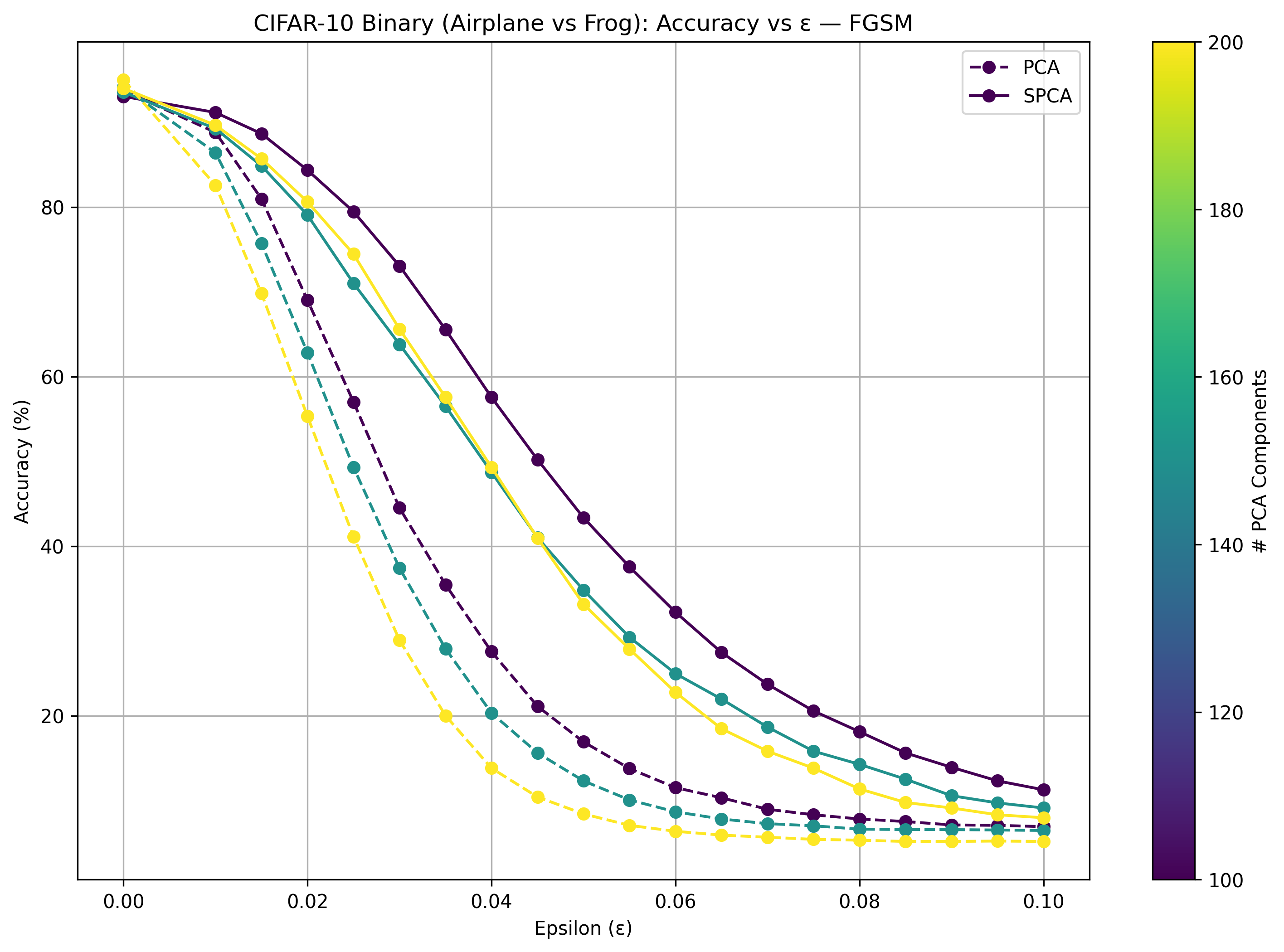}
        \caption{CIFAR‑Binary -- FGSM$_\infty$}
    \end{subfigure}
    \begin{subfigure}[b]{0.32\linewidth}
        \includegraphics[width=\linewidth]{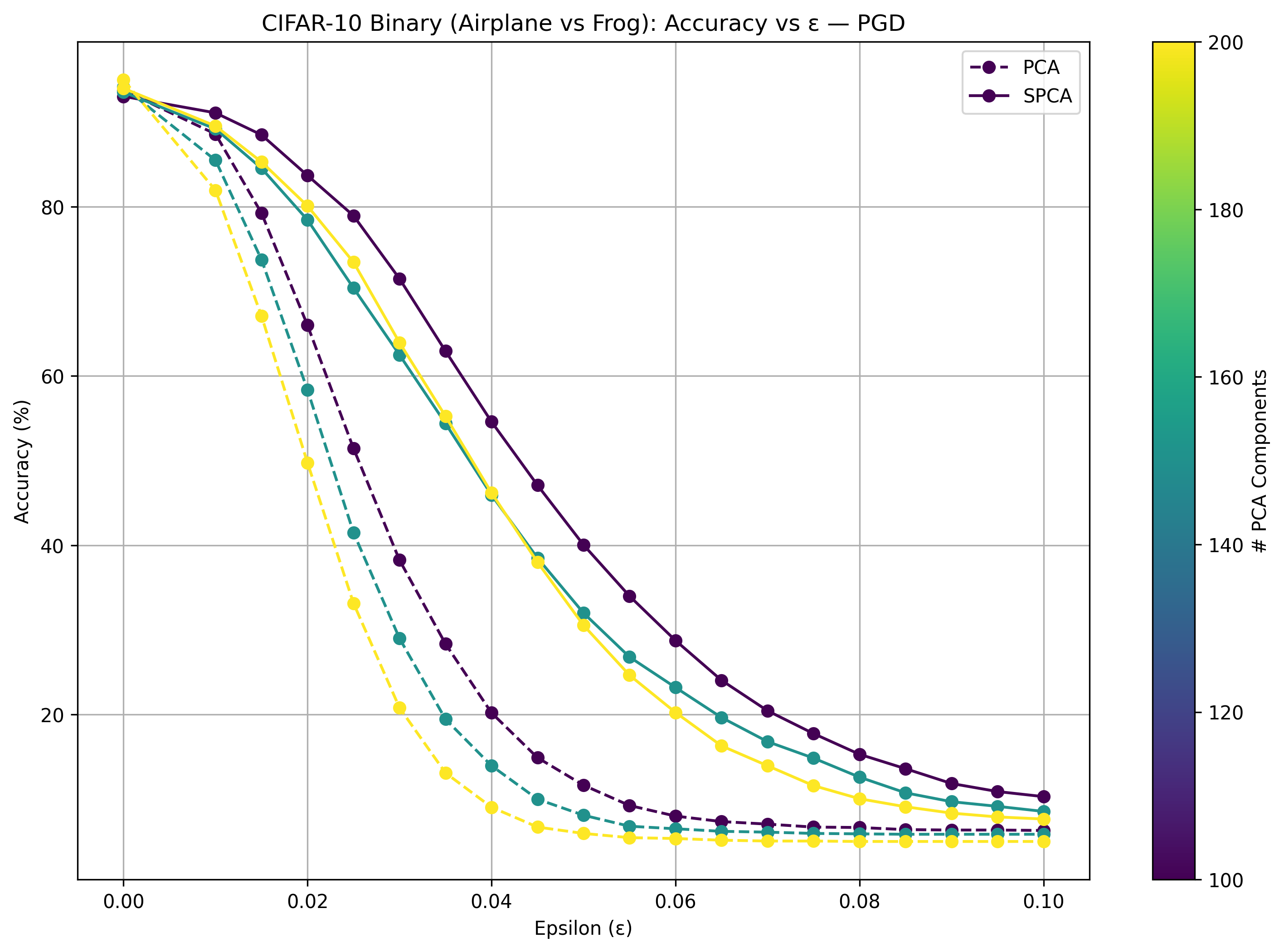}
        \caption{CIFAR‑Binary -- PGD$_\infty$}
    \end{subfigure}
    \begin{subfigure}[b]{0.32\linewidth}
        \includegraphics[width=\linewidth]{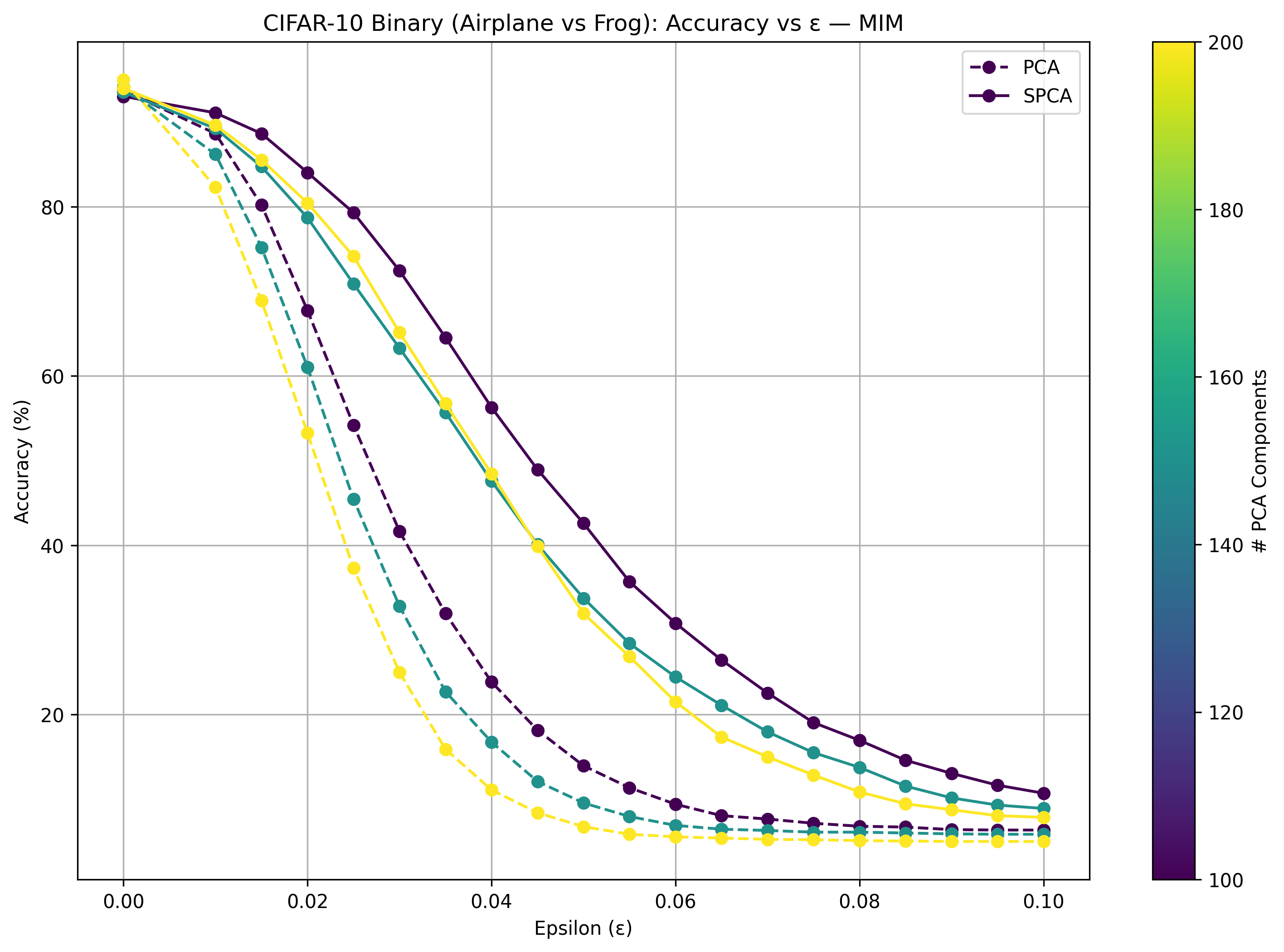}
        \caption{CIFAR‑Binary -- MIM$_\infty$}
    \end{subfigure}
    \caption{Classification accuracy of PCA‑ and SPCA‑based classifiers under white‑box attacks with $\ell_\infty$ perturbations.  Solid lines correspond to SPCA and dashed lines to PCA.  colors indicate the number of retained components (100-200).}
    \label{fig:linf_results}
\end{figure*}

On MNIST we observe that SPCA consistently outperforms PCA across all attacks and perturbation levels.  For example, under FGSM at $\varepsilon=0.1$ the PCA classifier’s accuracy drops below 80\% whereas the SPCA classifier remains above 95\%.  The gap widens with stronger attacks; at $\varepsilon=0.2$ SPCA still achieves around 70\% accuracy whereas PCA collapses near random chance.  This trend is even more pronounced for the iterative attacks PGD and MIM, reflecting the cumulative effect of repeated gradient steps.

The CIFAR‑Binary results exhibit similar patterns but with lower overall accuracy due to the greater complexity of natural images.  SPCA maintains a substantial advantage: accuracy curves decline more slowly and stay well above those of PCA for moderate $\varepsilon$.  We also note that increasing the number of retained components (shown by the color gradient) improves clean accuracy but slightly reduces robustness; however SPCA is less sensitive to this trade‑off than PCA.

\subsubsection{White‑box Attacks ($\ell_2$)}

Figure~\ref{fig:l2_results} shows analogous results for FGSM, PGD and MIM under $\ell_2$ constraints.  The overall shape of the curves differs from the $\ell_\infty$ case: accuracies remain high up to moderate values of $\varepsilon$ and decline more gradually.  On MNIST, SPCA and PCA perform similarly for small perturbations but diverge as $\varepsilon$ increases; the decrease in SPCA accuracy is slower, particularly under MIM.  On CIFAR‑Binary the distinction is clearer: SPCA retains high accuracy well beyond the point where PCA begins to degrade, confirming that sparsity improves robustness even for energy‑bounded perturbations.

\begin{figure*}[!ht]
    \centering
    \begin{subfigure}[b]{0.32\linewidth}
        \includegraphics[width=\linewidth]{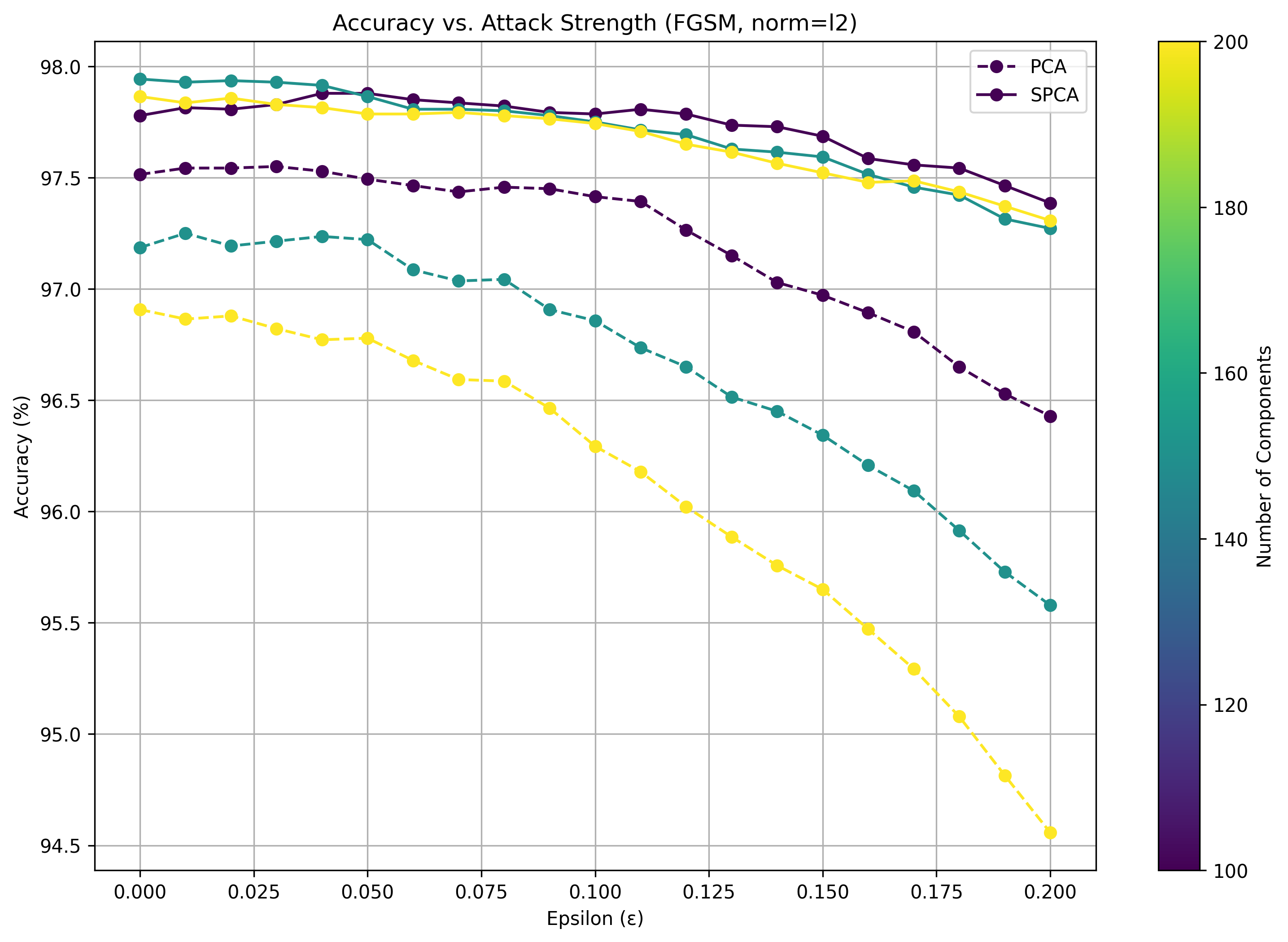}
        \caption{MNIST -- FGSM$_2$}
    \end{subfigure}
    \begin{subfigure}[b]{0.32\linewidth}
        \includegraphics[width=\linewidth]{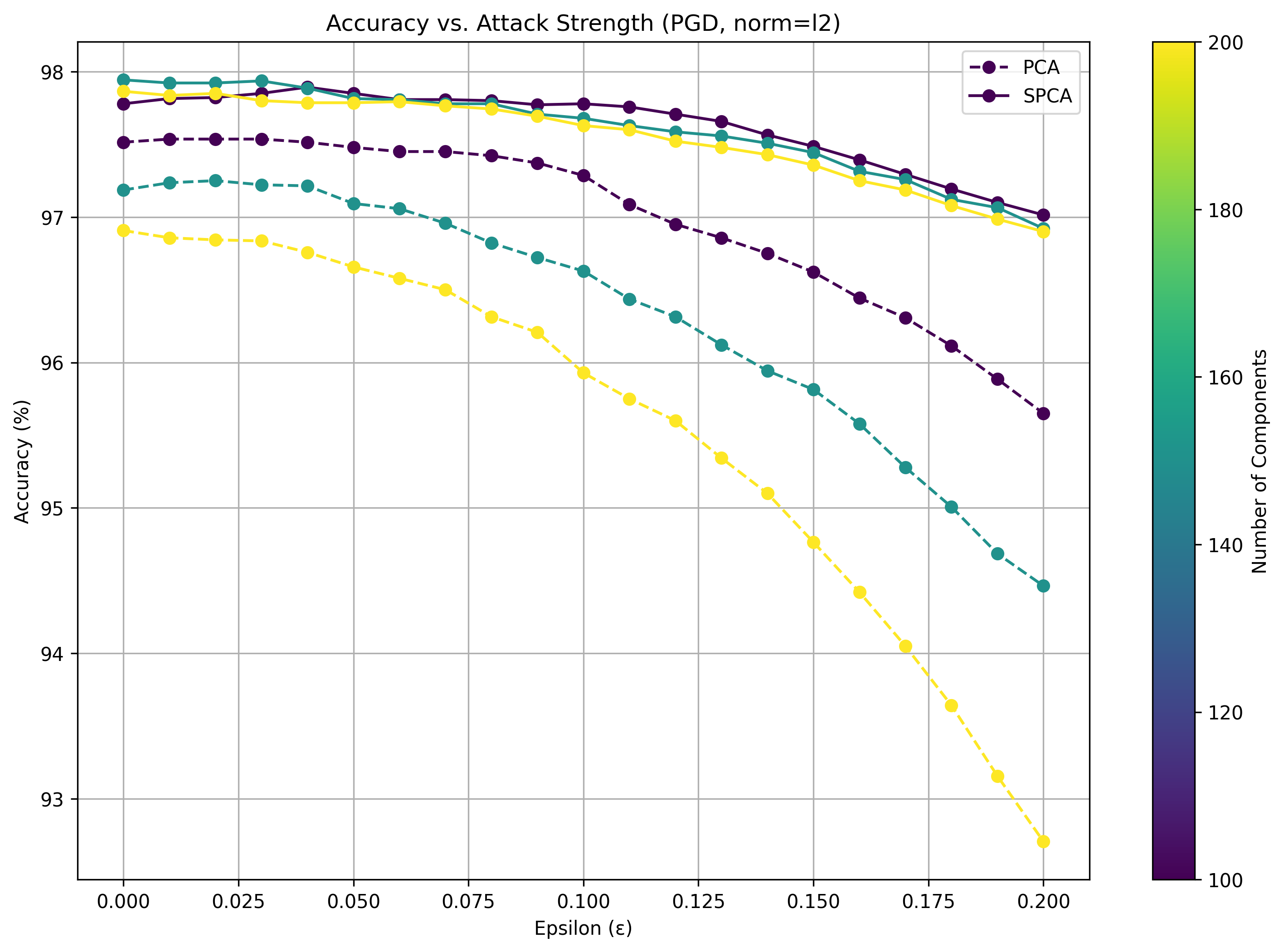}
        \caption{MNIST -- PGD$_2$}
    \end{subfigure}
    \begin{subfigure}[b]{0.32\linewidth}
        \includegraphics[width=\linewidth]{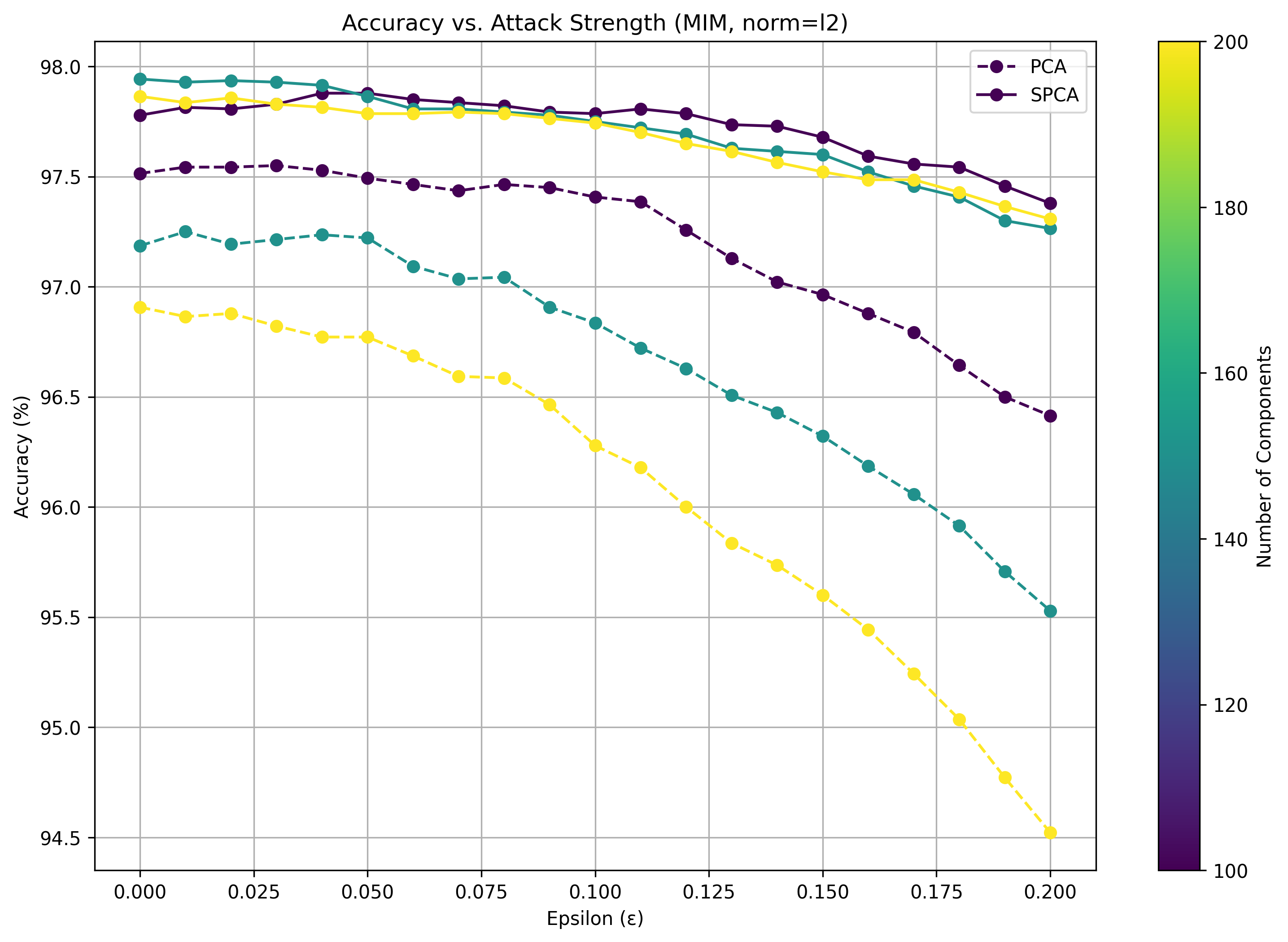}
        \caption{MNIST -- MIM$_2$}
    \end{subfigure}
    \\
    \begin{subfigure}[b]{0.32\linewidth}
        \includegraphics[width=\linewidth]{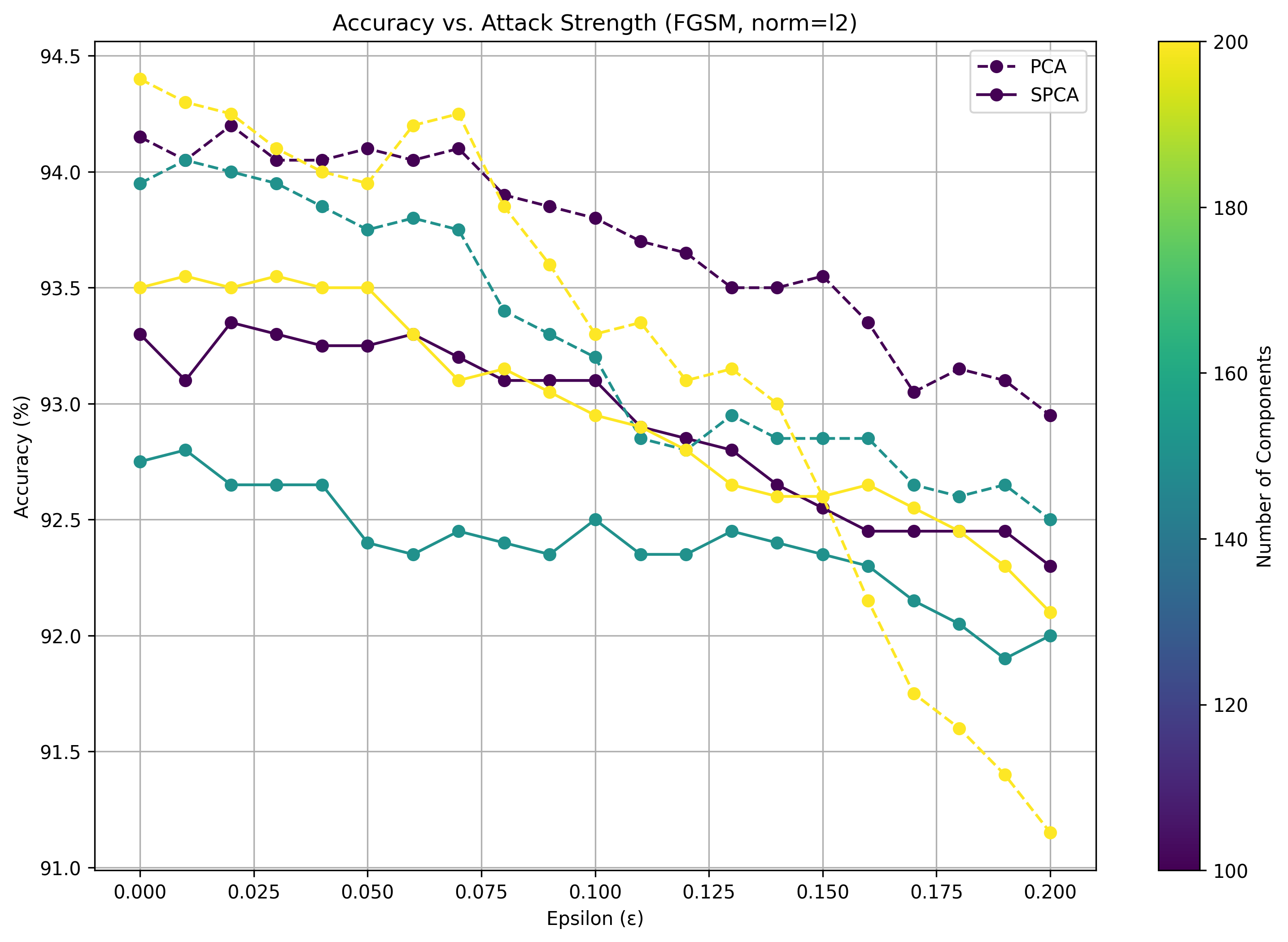}
        \caption{CIFAR‑Binary -- FGSM$_2$}
    \end{subfigure}
    \begin{subfigure}[b]{0.32\linewidth}
        \includegraphics[width=\linewidth]{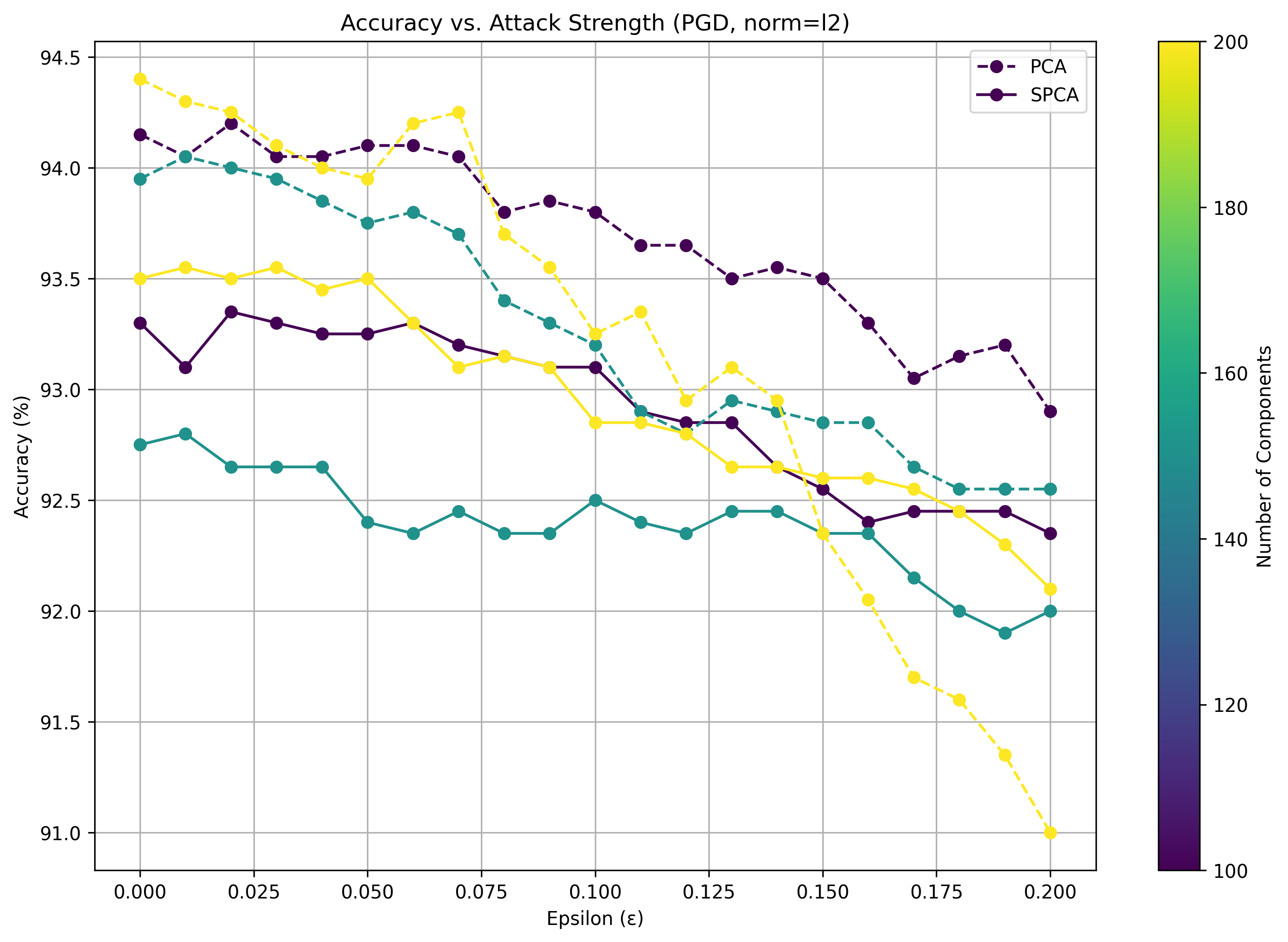}
        \caption{CIFAR‑Binary -- PGD$_2$}
    \end{subfigure}
    \begin{subfigure}[b]{0.32\linewidth}
        \includegraphics[width=\linewidth]{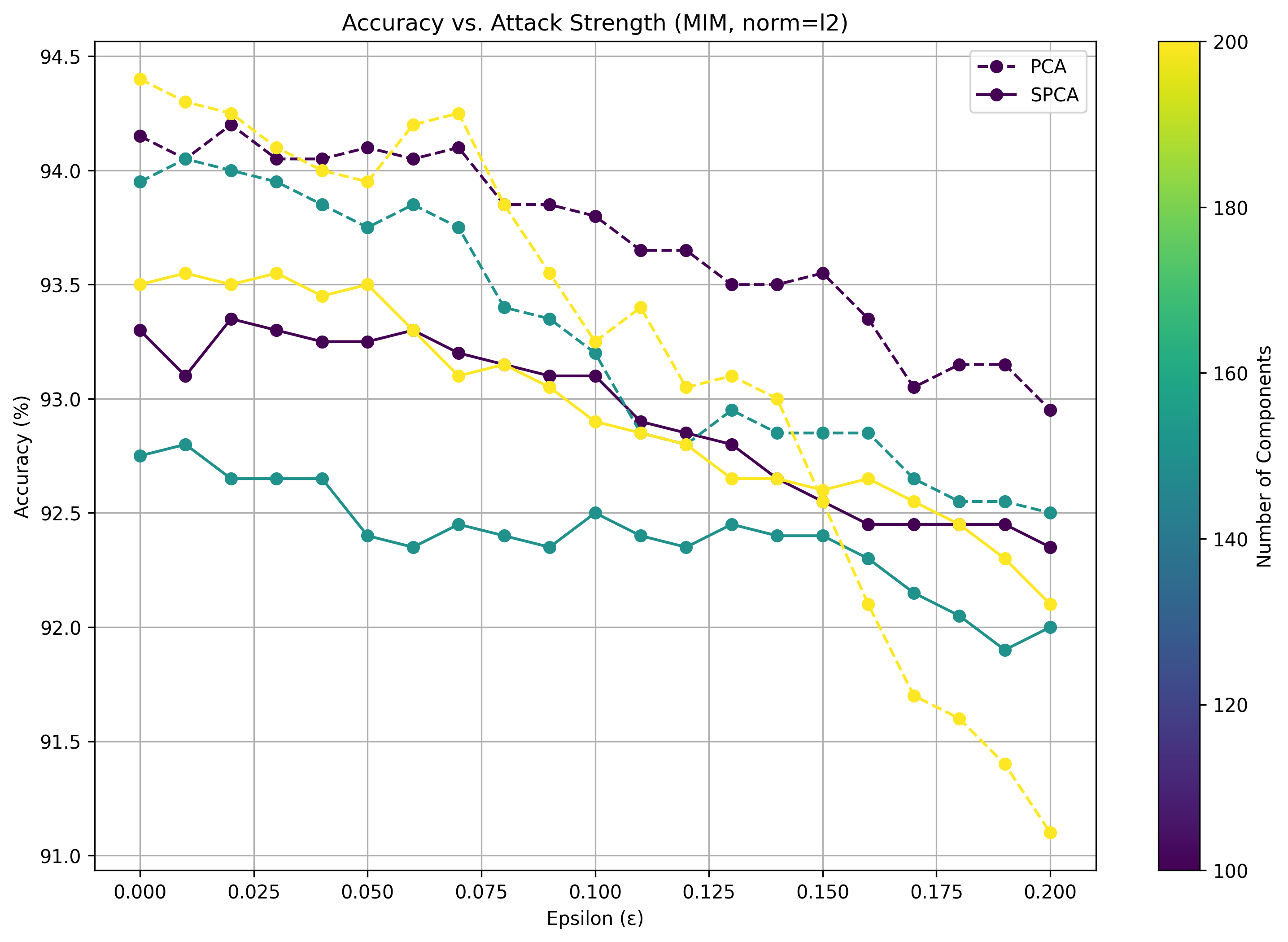}
        \caption{CIFAR‑Binary -- MIM$_2$}
    \end{subfigure}
    \caption{Classification accuracy of PCA‑ and SPCA‑based classifiers under white‑box attacks with $\ell_2$ perturbations.  Solid lines correspond to SPCA and dashed lines to PCA; colors indicate the number of retained components.}
    \label{fig:l2_results}
\end{figure*}

\subsubsection{Black‑box Attack}

Figure~\ref{fig:square_attack} reports results for the score‑based Square Attack on CIFAR‑Binary. SPCA maintains a slight advantage over PCA across the range of $\varepsilon$, with its accuracy decreasing less fast, especially for small $\varepsilon$. These results suggest that sparsity also benefits robustness against query‑efficient black‑box attacks.

\begin{figure}[!ht]
    \centering
    \begin{subfigure}[b]{0.49\linewidth}
        \includegraphics[width=\linewidth]{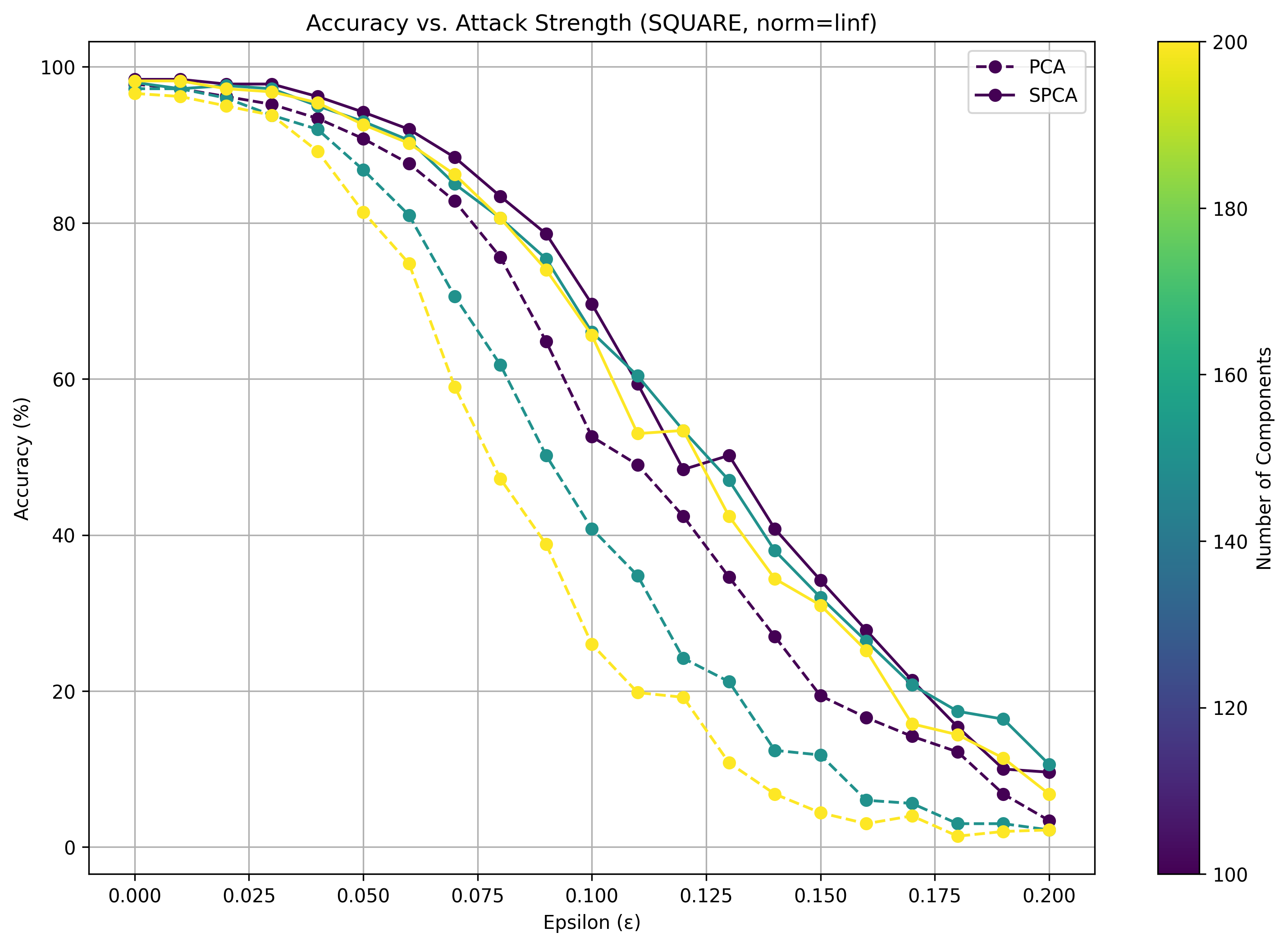}
        \caption{MNIST -- Square$_\infty$}
    \end{subfigure}
    \begin{subfigure}[b]{0.49\linewidth}
        \includegraphics[width=\linewidth]{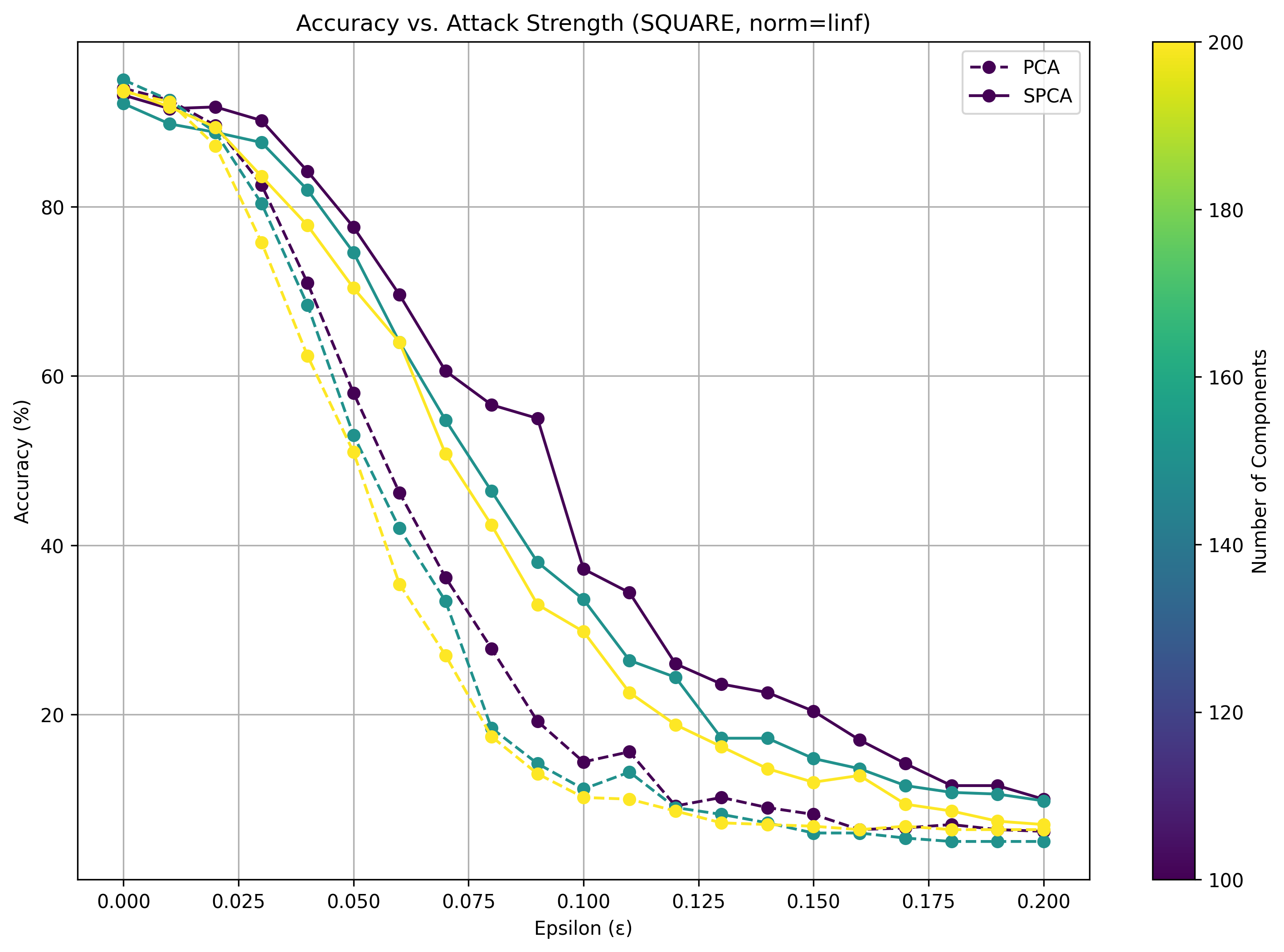}
        \caption{CIFAR-binary -- Square$_\infty$}
    \end{subfigure}
    \caption{Classification accuracy on MNIST and CIFAR‑Binary under the Square Attack. Solid lines correspond to SPCA; dashed lines to PCA; colors denote the number of retained components.}
    \label{fig:square_attack}
\end{figure}

\subsubsection{Visual Inspection of Adversarial Examples}

To better understand the nature of the perturbations, Figure~\ref{fig:adversarial_examples} displays examples of clean and adversarial images for both datasets at increasing $\varepsilon$. In the white‑box $\ell_\infty$ setting, perturbations manifest as high‑frequency noise distributed across the entire image; SPCA appears to mitigate their effect by discarding the corresponding high‑variance directions. In all cases, the perceptual quality of adversarial examples remains high even when the classifier is fooled.

\begin{figure}[!ht]
    \centering
    \includegraphics[width=\linewidth]{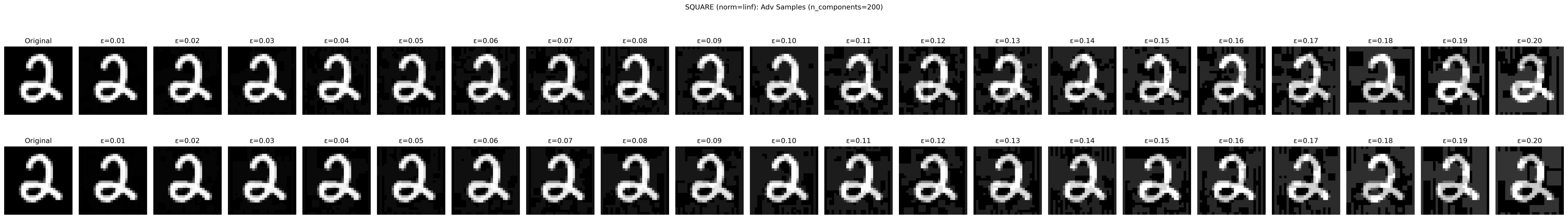}
    \includegraphics[width=\linewidth]{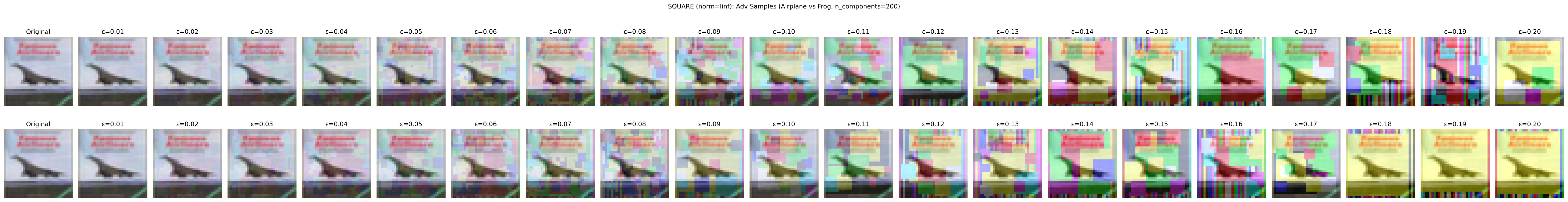}
    \includegraphics[width=\linewidth]{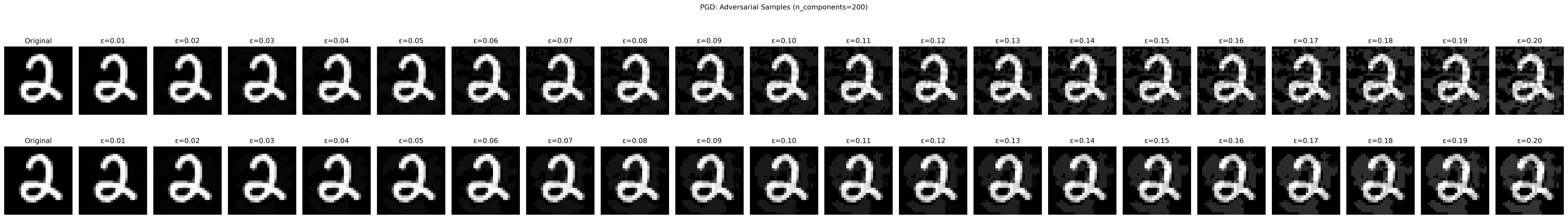}
    \includegraphics[width=\linewidth]{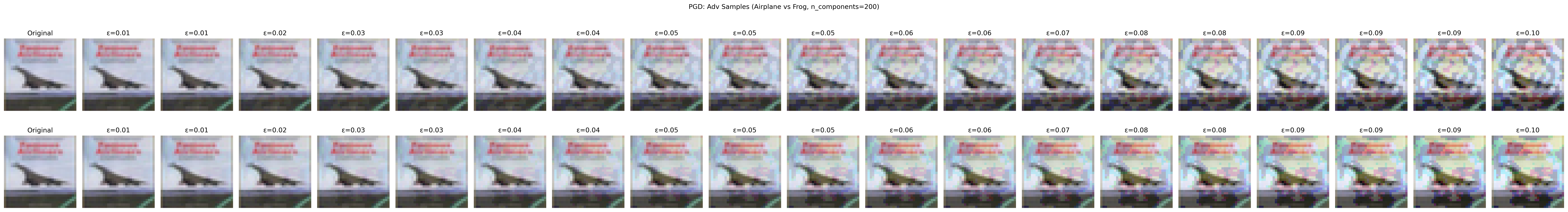}
    \includegraphics[width=\linewidth]{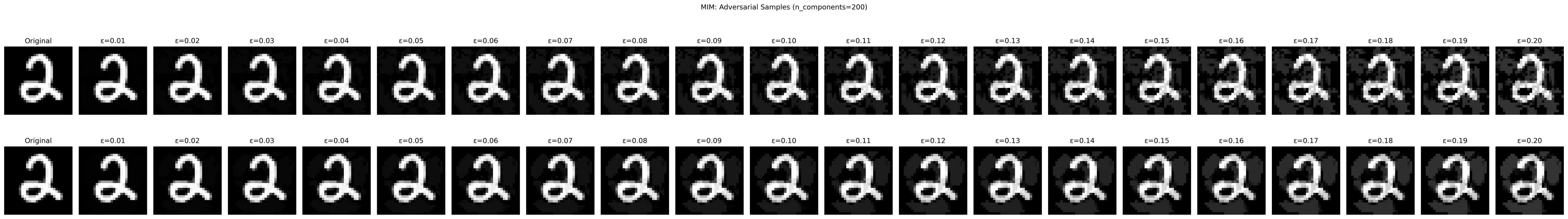}
    \includegraphics[width=\linewidth]{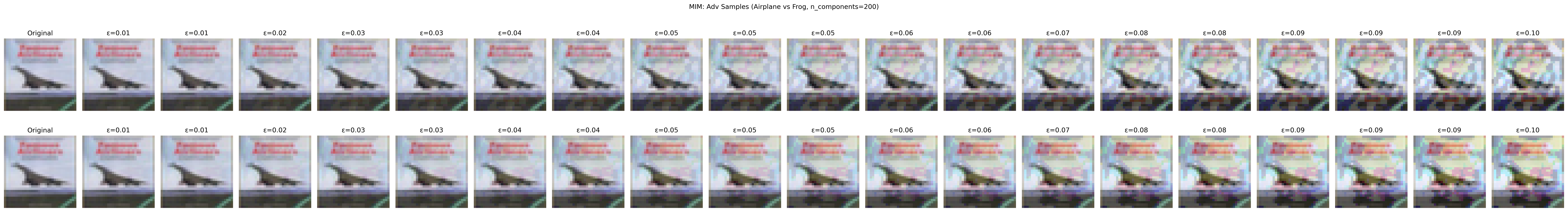}
    \caption{Representative clean (left columns) and adversarial (right columns) images for MNIST and CIFAR‑Binary across increasing $\varepsilon$. Rows correspond to different attack types. We visualize examples for models with 200 principal components. Best viewed zoomed in.}
    \label{fig:adversarial_examples}
\end{figure}

\section{Discussion}
Our experiments reveal several consistent trends. First, across datasets and attack types, SPCA-based classifiers uniformly outperform PCA-based ones. The advantage is most pronounced for $\ell_\infty$ attacks, where SPCA maintains high accuracy at perturbation levels that cause PCA to collapse; robustness under $\ell_2$ attacks declines more gradually for both methods, yet SPCA still provides a noticeable benefit, especially on CIFAR-binary. Second, the number of retained components mediates a trade-off between clean accuracy and robustness: retaining more components improves baseline accuracy but exposes the classifier to more potentially non-robust features. SPCA mitigates this trade-off, delivering stronger robustness at comparable or better clean accuracy for a fixed $r$. Finally, black-box results indicate that sparsity also helps when gradients are not available to the attacker, though the gap is smaller than in the white-box setting.

These empirical trends align closely with the theoretical mechanism established earlier. For linear heads, our certificates show that the per-example robust radius scales like a margin divided by a dual norm of $W^\top u$ (either $\ell_1$ or $\ell_2$). Enforcing sparsity in $W$ reduces column norms and drives many entries to zero, directly shrinking these dual norms and increasing the certified radius. For the small non-linear heads used in our experiments, the Lipschitz composition $\|C_\phi \circ W\|_{p\to 2} \le L_C \|W\|_{p\to 2}$ predicts the same direction of effect: sparser $W$ lowers operator-norm bounds and thus the model’s worst-case sensitivity, which is exactly what we observe empirically.

Scope and limitations deserve emphasis. The exact certificates apply to linear heads; while we verified the same robustness pattern with a small non-linear network, the non-linear guarantees are necessarily more conservative and come via Lipschitz bounds. The theory is margin-conditional: sparsifying $W$ improves the certified radius for a given margin but does not by itself guarantee larger margins. In practice, we observed that margins remain competitive so that the net effect is positive. Finally, our evaluation focuses on moderate projection sizes ($r \le 200$) and compact architectures; larger-scale settings may require additional optimisation care for SPCA.

\section{Conclusion}
\label{sec:conclusion}
This paper demonstrates that enforcing sparsity in linear feature extraction significantly improves the adversarial robustness of neural network classifiers. We introduced a theoretical framework that provides exact robustness certificates for linear heads on top of SPCA features in both $\ell_\infty$ and $\ell_2$ settings (binary and multiclass), and a complementary Lipschitz composition argument for general non-linear heads. The key mechanism is explicit: sparsity in the projection reduces dual/operator norms that control worst-case changes, thereby enlarging certified radii or tightening sensitivity bounds. Empirically, with a small non-linear network after the projection, SPCA consistently outperforms PCA across white-box and black-box attacks while maintaining competitive clean accuracy, confirming that the mechanism persists beyond the linear regime.

Several avenues remain for future work. Extending SPCA to higher-dimensional projections and larger architectures may require improved optimisation and structured sparsity. Combining sparse projections with adversarial training or certified methods (e.g., smoothing) could yield additive gains. It is also natural to explore adaptive sparsity patterns learned jointly with the classifier, and to evaluate on broader datasets and tasks. More broadly, our results suggest that shaping the front-end representation to contract adversarially exploitable directions is a principled and practical component for building robust systems.

% Acknowledgements and Disclosure of Funding should go at the end, before appendices and references

\acks{The authors conducted this research independently and did not receive any specific funding from public, commercial, or not-for-profit agencies.\\
The authors declare that they have no financial support or funding related to this work.
They also declare that they have no competing interests or other financial relationships that could be perceived to influence the results or interpretation of this article.\\
Further information about JMLR’s disclosure policies can be found on the journal’s website}


\newpage

% \appendix
% \section{}
% \label{app:theorem}

% % Note: in this sample, the section number is hard-coded in. Following
% % proper LaTeX conventions, it should properly be coded as a reference:

% %In this appendix we prove the following theorem from
% %Section~\ref{sec:textree-generalization}:

% In this appendix we prove the following theorem from
% Section~6.2:

% \noindent
% {\bf Theorem} {\it Let $u,v,w$ be discrete variables such that $v, w$ do
% not co-occur with $u$ (i.e., $u\neq0\;\Rightarrow \;v=w=0$ in a given
% dataset $\dataset$). Let $N_{v0},N_{w0}$ be the number of data points for
% which $v=0, w=0$ respectively, and let $I_{uv},I_{uw}$ be the
% respective empirical mutual information values based on the sample
% $\dataset$. Then
% \[
% 	N_{v0} \;>\; N_{w0}\;\;\Rightarrow\;\;I_{uv} \;\leq\;I_{uw}
% \]
% with equality only if $u$ is identically 0.} \hfill\BlackBox

% \section{}

% \noindent
% {\bf Proof}. We use the notation:
% \[
% P_v(i) \;=\;\frac{N_v^i}{N},\;\;\;i \neq 0;\;\;\;
% P_{v0}\;\equiv\;P_v(0)\; = \;1 - \sum_{i\neq 0}P_v(i).
% \]
% These values represent the (empirical) probabilities of $v$
% taking value $i\neq 0$ and 0 respectively.  Entropies will be denoted
% by $H$. We aim to show that $\fracpartial{I_{uv}}{P_{v0}} < 0$....\\

% {\noindent \em Remainder omitted in this sample. See http://www.jmlr.org/papers/ for full paper.}

\vskip 0.2in
\bibliography{main}

\end{document}